\documentclass[11pt]{article}

\usepackage{underscore}
\usepackage{array}
\newcolumntype{P}[1]{>{\centering\arraybackslash}p{#1}}
\newcolumntype{M}[1]{>{\centering\arraybackslash}m{#1}}
\usepackage[applemac]{inputenc}
\usepackage[numbers,compress]{natbib}
\usepackage{amsthm}
\usepackage{amsmath}
\usepackage{algcompatible}
\usepackage{algorithm}
\usepackage{booktabs}
\usepackage{enumerate}
\usepackage{graphicx}
\usepackage{amssymb}
\usepackage{latexsym}
\usepackage{epstopdf}
\usepackage{color}
\usepackage{xcolor}
\usepackage{bbm}
\usepackage{needspace}
\usepackage{color, colortbl}
\usepackage[english]{babel}
\usepackage{tikz}
\usepackage{caption}
\usepackage{subcaption}
\usetikzlibrary{arrows,automata}
\usetikzlibrary{positioning}
\usepackage{filecontents}
\usepackage{mathtools}
\usepackage{algorithm}
\usepackage{multirow}
\usepackage{algpseudocode}
\usepackage{appendix} 
\usepackage{fullpage}   

\usepackage{arydshln}

\DeclareCaptionFont{mysize}{\fontsize{8.0}{9.6}\selectfont}
\usepackage{ulem}

\usepackage{todonotes}

\DeclareCaptionFont{mysize}{\fontsize{8.0}{9.6}\selectfont}
\theoremstyle{plain}
\newtheorem{thm}{\protect\theoremname}
\theoremstyle{plain}

\newtheorem{remark}{Remark}
\newtheorem{coro}{Corollary}

\newtheorem{define}{Definition}

\newtheorem{assume}{Assumption}

\makeatother

\usepackage{babel}
\providecommand{\lemmaname}{Lemma}
\providecommand{\theoremname}{Theorem}

\newcommand{\prox}{\mathrm{prox}}

\begin{document}
\title{\vspace{0.1in} FedADMM: A Federated Primal-Dual Algorithm Allowing Partial Participation}
\author{Han Wang, Siddartha Marella, James Anderson%
\thanks{ This work is supported by awards from the NSF; award ID CAREER-2144634, and DOE; award ID DE-SC0022234.
The authors are with the Department of Electrical Engineering, Columbia University in the City of New York, New York, NY, 10027, USA (e-mail \texttt{\{hw2786, sm4940, james.anderson\}@columbia.edu).}}  }

\maketitle
\normalem 
\begin{abstract}
Federated learning is a framework for distributed optimization that places emphasis on communication efficiency. In particular, it follows a client-server broadcast model and is particularly appealing because of its ability to accommodate  heterogeneity in client compute and storage resources, non-i.i.d. data assumptions, and data privacy.
Our contribution is to offer a new federated learning algorithm, FedADMM, for solving non-convex composite optimization problems with non-smooth regularizers. We prove converges of FedADMM for the case when not all clients are able to participate in a given communication round under a very general sampling model.
\end{abstract}
\section{Introduction}
Federated learning 
(FL)~\cite{konevcny2016federated,mcmahan2017communication}, a novel distributed learning paradigm, has attracted significant attention in the past few years. Federated algorithms take a client/server computation model, and provide scope to  train large-scale machine learning models over an edge-based distributed computing architecture. In the paradigm of FL, models are trained collaboratively under the coordination of a central server while storing data locally on the edge/clients. Typically, clients (devices and entities ranging from mobile phones to hospitals, to an internet of things~\cite{mohri2019agnostic,kairouz2021advances}) are assumed to be heterogeneous; each client is subject to its own constraints on available computational and storage resources. By allowing data to be stored client-side, the FL paradigm has many favorable privacy properties.

In contrast to ``traditional'' distributed optimization, FL framework has its own unique challenges and characteristics. First, \textit{communication} becomes problematic when the number of edge devices/clients is large, or
the connection between the central server and a device is slow, e.g., when the mobile phones have limited bandwidth. Second, datasets stored in each client may be highly \textit{heterogeneous} in that they are sampled from different population distributions, or the amount of data belonging to each client is unbalanced. Third,  \textit{device/client heterogeneity} can severely hinder algorithm performance; differences in hardware, software, and power (connectivity) lead to varying computation speeds among clients, leading to global performance being dominated by the slowest agent. This is known as the ``straggler'' effect. Additionally, the server  may lose control over the clients when they power down or lose connectivity. It is thus common for only a fraction of clients to participate in in each round of the training (optimization) process, and federated optimization algorithms must accommodate this \textit{partial participation}. 

A wealth of algorithms have been developed to address the aforementioned challenges. Notably, work in \cite{mcmahan2017communication} proposed the now popular FedAvg algorithm, where each client performs multiple stochastic gradient descent (SGD) steps before sending the model to the server for aggregation. Subsequent efforts  \cite{konevcny2016federated,stich2018local,wang2018cooperative, li2019convergence, reisizadeh2020fedpaq}
provided theoretical analysis and further empirical performance evaluations.
 Since the proposal of FedAvg, there has been a rich body of work concentrating on developing  federated optimization algorithms, such as; FedProx~\cite{sahu2018convergence}, FedSplit~\cite{pathak2020fedsplit}, Scaffold~\cite{karimireddy2020scaffold}, FedLin~\cite{mitra2021linear}, FedDyn~\cite{acar2021federated}, FedDR~\cite{tran2021feddr} and FedPD~\cite{zhang2021fedpd}.

We consider a general unconstrained, composite optimization models formulated as
\begin{equation}\label{eq:opt}
    \frac{1}{n}\sum_{i=1}^nf_i(x)+g(x).
\end{equation}
No convexity assumptions on $f_i$ are made and $g$ can be non-smooth. Of the previously mentioned federated algorithms, we restrict our attention to FedDR and FedPD. These algorithms are designed to alleviate the unrealistic assumptions required by FedAvg in order to realize desirable theoretical convergence rates. As described in~\cite{tran2021feddr}, FedDR  combines the nonconvex Douglas-Rachford splitting (DRS) algorithm~\cite{li2016douglas} with a randomized block-coordinate strategy. FedDR provably converges when only a subset of clients participate in any given communication round. In contrast, FedPD is a primal-dual algorithm which requires either full participation or no participation by all clients at every per round. Unlike FedDR,  FedPD cannot handle optimization problems of the form of~\eqref{eq:opt} for $g \not\equiv 0$.

The key observation of this paper is to note that the updating rules of FedPD share a similar form to those of the alternating direction
method of multipliers (ADMM)~\cite{gabay1976dual}, but specifies how the local models are updated to satisfy the flexibility need of FL. Motivated by the fact that ADMM is  the dual formulation of DRS~\cite{yan2016self,fukushima1992application}, we provide a new algorithm called FedADMM. Specifically, our contributions are:
\begin{enumerate}
    \item By applying FedDR to the dual formulation of problem~\eqref{eq:opt}, we propose a new algorithm called FedADMM, which allows partial participation and solves the federated composite optimization problems as in~\cite{yuan2021federated}.
    \item When $g \equiv 0$ in problem~\eqref{eq:opt}, we find that FedADMM reduces to FedPD but requires only partial participation. 
      \item We prove equivalence between FedDR and FedADMM and provide a one-to-one and onto mapping between the the iterates of both algorithms.
      \item We provide  convergence guarantees for FedADMM using the equivalence established in point 3.
\end{enumerate}
Since FedADMM is the dual formulation of FedDR, it inherits all the desirable properties from FedDR. First, it can handle both \textit{statistical} and \textit{system} heterogeneity. Second, it allows inexact evaluation of users' proximal operators as in FedProx and FedPD.
Third, by considering  $g\not\equiv 0$ in ~\eqref{eq:opt}, more general applications and problems with constraints can be considered~\cite{yuan2021federated}.
\subsection{Related Work}
\textbf{ADMM and DRS:}
DRS was first proposed in~\cite{douglas1956numerical} in the context of providing numerical solutions to heat conduction partial differential equations. Subsequently, it found applications in the solution of convex optimization problems~\cite{lions1979splitting,seidman2019control} and later non-convex problems~\cite{li2016douglas,li2015global,themelis2020douglas}. ADMM~\cite{giselsson2016linear,boyd2011distributed} is a very popular iterative algorithm for solving composite optimization problems. The equivalence between DRS and ADMM has been subject of a lot of work~\cite{gabay1983chapter,eckstein1989splitting,yan2016self,zhao2021automatic}.
It was first established for convex problems where ADMM is equivalent to applying DRS to the dual problem~\cite{gabay1983chapter,eckstein1989splitting}. Recently, these ideas were extended in \cite{themelis2020douglas} to show  equivalence in the  non-convex regime. Inspired by the fact that FedDR can be viewed as a variant of nonconvex DRS applied to the FL framework, we propose a new algorithm, FedADMM  and further extend the equivalence of these two algorithms to the FL paradigm.

\textbf{Federated Learning:} FedAvg was first proposed in~\cite{mcmahan2017communication}.
However, it works well only with a homogeneous set of clients. It is  difficult to analyze the convergence of FedAvg for the heterogeneous setting unless additional assumptions are made~\cite{li2019convergence,li2019convergence[b],khaled2019first,wang2019adaptive}.
The main reason for this is that the algorithm suffers from client-drift~\cite{zhao2018federated} under objective heterogeneity. To address the data and system heterogeneity, FedProx~\cite{sahu2018convergence} was proposed by adding an extra proximal term~\cite{ParB14} to the objective. However, this extra term might degrade the training performance so that FedProx doesn't converge to the global or local stationery points unless the step-size is carefully tuned. Another method called Scaffold~\cite{karimireddy2020scaffold} uses control variates (or variance reduction) to reduce client-drift at the cost of increased communication incurred by sending extra variables to the server. FedSplit~\cite{pathak2020fedsplit} applied the operator splitting schemes to remedy the objective heterogeneity issues, while it only considered the convex problems and required the full participation of clients. As mentioned earlier, FedDR~\cite{tran2021feddr} was  inspired from  DRS, and allowed partial participation. From the primal-dual optimization perspective, FedPD \cite{zhang2021fedpd}  proposed  a new concept of participation, which restricted its potential application on real problems.
It is also worthwhile to mention that FedDyn~\cite{acar2021federated} is equivalent to FedPD~\cite{zhang2021fedpd} from~\cite{zhang2021connection} under the full participation setting, but it allows partial participation. Unlike~\cite{yuan2021federated}, FedPD and FedDyn can't solve  non-smooth or constrained problems. Finally, we refer readers to~\cite{kairouz2021advances} for a comprehensive understanding of the recent advances in FL.

\section{Preliminaries and Problem Formulation}
We consider the canonical Federated learning optimization problem defined as
\begin{equation}\label{eqobjective}
\min _{x \in \mathbb{R}^{d}}\left\{F(x)=f(x)+g(x)\equiv\frac{1}{n} \sum_{i=1}^{n} f_{i}(x)+g(x)\right\}
\end{equation}
where $n$ is the number of clients,  $f_{i}$ denotes the loss function associated to  the $i$-th client. Each $f_i$ is nonconvex and Lipschitz differentiable (see Assumptions $2.1$ and $2.2$ below), and $g$ is a proper, closed, and convex function and is not necessarily smooth. For example, $g$ could be any $\ell_p$ norm or an indicator function.

\begin{assume}\label{assume1}
$F(x)$ is bounded below, i.e., $$
\inf _{x \in \mathbb{R}^{d}} F(x)>-\infty
\text{ and } \text{dom}(F)\neq  \emptyset.$$
\end{assume}
\begin{assume}\label{assume2}\textbf{(Lipschitz differentiability)} Each $f_{i}(\cdot)$ in~\eqref{eqobjective} has $L$-Lipschitz gradient, i.e.,
$$
\left\|\nabla f_{i}(x)-\nabla f_{i}(y)\right\| \leq L\|{x}-{y}\| 
$$
for all  $ i \in[n]$ and $x, y \in \mathbb{R}^{d}$.
 \end{assume}
The notation $ [n]$ above defines the set $\{1,2,\hdots,n\}$. All the norms in the paper are $\ell_2$ norm.
We will frequently make use of the proximal operator~\cite{ParB14}. Although typically defined for convex functions, we make no such assumptions.
 \begin{define}
 \label{defprox} \textbf{(Proximal operator)} Given an $L$-Lipschitz (possibly nonconvex and nonsmooth) function $f$, then the proximal mapping $ \mathbb{R}^d  \rightarrow (-\infty,\infty]$ is defined as
    \begin{equation}
    \begin{aligned}
    \prox_{\eta f} (x)=\arg\min_{y} \left\{f(y)+\frac{1}{2\eta} \lVert x-y \rVert^2 \right\}.
    \end{aligned}
    \end{equation}	
    where parameter $\eta>0$. 
    \end{define}
If $f$ is nonconvex but $L$-Lipschitz, $\prox_{\eta f} (x)$ is still well-defined with $0<\eta<1/L.$    
\begin{define}
\textbf{(Conjugate function)} Let $f: \mathbb{R}^{d} \rightarrow \mathbb{R}$. The function $f^{*} :\mathbb{R}^{d} \rightarrow \mathbb{R}$ defined as
$$
f^{*}(y) \triangleq \sup _{x \in \operatorname{dom} f}\left(y^{T} x-f(x)\right)
$$
is called the conjugate function of $f$.
\end{define}
Note that the conjugate function is closed and convex even when $f$ is not, since it is the piecewise supremum of a set of affine functions.

\begin{define}\label{station}\textbf{($\varepsilon$-stationarity)}
A vector $x$  is said to be an $\varepsilon$-stationery solution to~\eqref{eqobjective} if $$
\mathbb{E}\left[\left\|\nabla{F}(x)\right\|^{2}\right] \leq \varepsilon^{2},
$$
where expectation is taken with respect to all  random variables in the respective algorithm.
\end{define}

\section{Douglas-Rachford Algorithm}\label{section_DRS}
\subsection{Douglas-Rachford Splitting}
Douglas-Rachford Splitting (DRS)~\cite{douglas1956numerical} is an iterative splitting  algorithm for solving the optimization problems that can be written as
\begin{equation}\label{eqoptimize}
\operatorname{minimize}_{x \in 
\mathbb{R}^{d}}\quad f(x)+g(x).
\end{equation}
Although originally used for solving convex problems, it has been shown to work well on certain non-convex problems with additional structure.
DRS solves problem~\eqref{eqoptimize}
by producing a series of iterates $(y_k,z_k,x_k)$ for $k=1,2,\hdots$ given by
\begin{equation}\label{eqDRS_it}
\begin{cases}
y_{k} &=\operatorname{prox}_{\eta f}\left(x_{k}\right) \\
z_{k} &=\operatorname{prox}_{\eta g}\left(2 y_{k}-x_{k}\right) \\
x_{k+1} &=x_{k}+\alpha\left(z_{k}-y_{k}\right)
\end{cases}
\end{equation}
where $\alpha$ is a relaxation parameter. When $\alpha=1$, \eqref{eqDRS_it} is the classical Douglas-Rachford splitting and when $\alpha=2$,  \eqref{eqDRS_it} is a related splitting algorithm called  Peaceman-Rachford splitting~\cite{peaceman1955numerical}.

If $f$ in problem~\eqref{eqoptimize} can decomposed as $f(x)=\frac{1}{n} \sum_{i=1}^{n} f_{i}(x)$, then~\eqref{eqDRS_it} can be modified so as to run in parallel if we include a global averaging step. The resulting algorithm is given below:
\begin{equation}\label{eq:parrallel}
\begin{cases}
y_{i}^{k+1} &=y_{i}^{k}+\alpha\left(\bar{x}^{k}-x_{i}^{k}\right), \quad \forall i \in[n] \\
x_{i}^{k+1} &=\operatorname{prox}_{\eta f_{i}}\left(y_{i}^{k+1}\right), \quad \forall i \in[n] \\
\hat{x}_{i}^{k+1} &=2 x_{i}^{k+1}-y_{i}^{k+1}, \quad \forall i \in[n] \\
\tilde{x}^{k+1} &=\frac{1}{n} \sum_{i=1}^{n} \hat{x}_{i}^{k+1}, \\
\bar{x}^{k+1} &=\operatorname{prox}_{\eta g}\left(\tilde{x}^{k+1}\right) .
\end{cases}
\end{equation}
A full derivation is given in~\cite{tran2021feddr}. Equation~\eqref{eq:parrallel} is called  full parallel Douglas-Rachford splitting (DRS).
\subsection{FedDR}
Implicit in the full parallel DRS~\eqref{eq:parrallel}, is the fact that all users are required to participate at every iteration. Instead of requiring all users $i \in[n]$ to participate as in~\eqref{eq:parrallel}, work in~\cite{tran2021feddr} proposed an inexact randomized block-coordinate DRS algorithm, called FedDR. Here, a subset $\mathcal{S}_{k}$ of clients is sampled from a ``proper'' sampling scheme $\hat{\mathcal{S}}$ (See Definition~\ref{proper} below for details) at each iteration. Each client,  $i\in \mathcal{S}_{k}$ performs a local update (i.e., executes the  first three steps in~\eqref{eq:parrallel}), then sends its local model to server for aggregation. Each client $i \notin \mathcal{S}_{k}$ does noting. The complete FedDR algorithm is shown in Alg~\ref{algFedDR}.

\begin{algorithm}
\caption{FL with Randomized DR (FedDR)~\cite{tran2021feddr}} \label{algFedDR}
\begin{algorithmic}[1]
\State \textbf{Initialize}  $
 x^{0}, \eta, \alpha>0, K, \text { and tolerances } \epsilon_{i, 0} \geq 0.
$\\
\textbf{Initialize} the server with $\bar{x}^{0}=x^{0}$ and $\tilde{x}^{0}=x^{0}$\\
\textbf{Initialize} each client $i\in[n]$ with $y_{i}^{0}=x^{0}, x_{i}^{0} \approx \operatorname{prox}_{n f_{i}}\left(y_{i}^{0}\right), \text { and } \hat{x}_{i}^{0}=2 x_{i}^{0}-y_{i}^{0} \text {. }
$
\State \textbf{for} $k=0, \ldots, K$ \textbf{ do } \
\State \quad $\text {Randomly sample }\mathcal{S}_k \subseteq[n] \text{ with size } S.
$
\State \quad $\rhd$ {User side}
\State \quad \textbf{for each user} $i \in \mathcal{S}_{k}$ \textbf{do} 
\State \quad \quad receive $\bar{x}^{k}$ from the server.
\State \quad \quad choose $\epsilon_{i, k+1} \geq 0$ and update
\State \quad \quad$y_{i}^{k+1}=y_{i}^{k}+\alpha\left(\bar{x}^{k}-x_{i}^{k}\right),$ 
\State \quad \quad $x_{i}^{k+1} \approx \operatorname{prox}_{\eta f_{i}}\left(y_{i}^{k+1}\right),$
\State\quad\quad $\hat{x}_{i}^{k+1}=2 x_{i}^{k+1}-y_{i}^{k+1}.$
\State \quad \quad send  $\Delta \hat{x}_{i}^{k}=\hat{x}_{i}^{k+1}-\hat{x}_{i}^{k} \text { back to the server }.$
\State \quad \textbf{end for}\
\State \quad $\rhd$ {Server side}
\State \quad aggregation $\tilde{x}^{k+1}=\tilde{x}^{k}+\frac{1}{n} \sum_{i \in \mathcal{S}_{k}} \Delta \hat{x}_{i}^{k}$\
\State \quad update $\bar{x}^{k+1}=\operatorname{prox}_{\eta g}\left(\tilde{x}^{k+1}\right)$\
\State\textbf{end for}
\end{algorithmic}
\end{algorithm}

Convergence to an $\epsilon$-stationary point of FedDR is guaranteed when the sampling scheme $\hat{\mathcal{S}}$ is proper and Assumption~\ref{assume1} and~\ref{assume2} hold~\cite{tran2021feddr}.
\begin{define}\label{proper}
Let $p=(p_1, p_2, \cdots, p_n),$ where $p_i =\mathbb{P}(i \in \hat{\mathcal{S}})$. If $p_i>0$ for all $i \in[n]$, we call the sampling scheme $\hat{\mathcal{S}}$ proper, i.e., every client has a nonzero probability to be selected.
\end{define}
\begin{assume}\label{sampling}
All partial participation algorithms in this paper use a proper sampling scheme.
\end{assume}

From the analysis in~\cite{richtarik2016parallel}, this assumption includes a lot of sampling schemes such as non-overlapping uniform and doubly uniform sampling as special cases. The intuition behind proper sampling is to ensure that on average every client has a chance to be selected at every iteration.

In FedDR  there are three variables that get updated:   $\bar{x}^k, x_i^k$ and $y_i^k$. The variable $\bar{x}^k$ denotes the consensus/average variable to minimize the global model $F$, $x_i^k$ denotes the local variable associated to $f_i,$ while $y_i^k$ measures the distance between the global variable $\bar{x}^k$ and local model $x_i^k.$ To account for the limitations on computation resources for local users, FedDR allows the inexact calculation of the proximal step, i.e., \begin{equation*}
x_{i}^{k+1} \approx \operatorname{prox}_{\eta f_{i}}\left(y_{i}^{k+1}\right) \iff 
 \left\| x_i^{k+1} - \operatorname{prox}_{\eta f_{i}}\left(y_{i}^{k+1}\right)\right\|  \le \epsilon_{i,k+1}.
\end{equation*}
Thus $\approx$ defines an $\epsilon$-close solution.
After local clients $i\in \mathcal{S}_k$ update their model and send them back to the server, the server aggregates the updates to update the global model by executing line 16 and 17 in Algorithm~\ref{algFedDR}.

\section{From FedDR to FedADMM}
Our first contribution is to derive the FedADMM algorithm from FedDR. 

\subsection{An equivalent formulation}
We begin by rewriting problem \eqref{eqobjective} as the equivalent constrained problem:
\begin{equation}\label{eqform2}
\begin{aligned}
\min_{x\in \mathbb{R}^{n d},\bar{x}}  \left\{ F(x) = \frac{1}{n} \sum_{i=1}^{n} f_{i}(x_{i})+g(\bar x) \right\} \\
 \text { s.t. }  \mathbb{I}_{nd}  x=\mathbbm{1}\bar{x}  \hspace{3cm}
\end{aligned}
\end{equation}
where $x=\left[x_{1}^T, x_{2}^T, \cdots, x_{n}^T\right]^T \in \mathbb{R}^{n d}$, $\mathbb{I}_{d}$ is the $d\times d$ identity matrix, and $\mathbbm{1} = [\mathbb{I}_d \ \cdots \ \mathbb{I}_d]^T$. Here $\bar{x}$ should be interpreted as the global consensus variable.

Forming the Lagrangian of~\eqref{eqform2} and using the definition of the conjugate function, the dual formulation of~\eqref{eqform2} is
\begin{equation}\label{eqdual_form}
\max _{z\in \mathbb{R}^{nd}} \left\{ F^{*}(z)=-f^{*}(-\mathbb{I}_{n d}z)-g^{*}(\mathbbm{1}^T z) \right\}
\end{equation}
where $z=\left[z_{1}^T, z_{2}^T, \cdots, z_{n}^T\right]^T \in \mathbb{R}^{n d}$ is the vector of dual variables. Problem \eqref{eqdual_form} is clearly equivalent to 
\begin{equation}\label{eqdual2}
\min_{z_1,z_2,\cdots,z_n} \left\{\frac{1}{n} \sum_{i=1}^{n} f^{*}_{i}\left(-z_{i}\right)+g^{*}\left(\sum_i^n{z_i}\right)\right\}.
\end{equation}

Before proceeding to develop an algorithm for solving~\eqref{eqdual2}, we first rewrite the full parallel DRS algorithm~\ref{eq:parrallel}. Changing the execution order of~\eqref{eq:parrallel} and choosing $\alpha=1$ give
\begin{equation}\label{eqDR2}
\begin{cases}\hat{x}_{i}^{k} & =2 x_{i}^{k}-y_{i}^{k}, \quad \forall i \in[n] \\
\tilde{x}^{k} & =\frac{1}{n} \sum_{i=1}^{n} \hat{x}_{i}^{k},\quad \forall i \in[n] \\ \bar{x}^{k} & =\operatorname{prox}_{\eta g}\left(\tilde{x}^{k}\right),\\ 
x_{i}^{k+1} & =\operatorname{prox}_{\eta f_{i}}\left(y_{i}^{k}+\bar{x}^k -x_i^k\right), \quad \forall i \in[n]\\
y_{i}^{k+1} & =y_{i}^{k}+\bar{x}^{k}-x_{i}^{k}, \quad \forall i \in[n]. \end{cases}
\end{equation}
Introducing the change of variables $w_i^{k} =x_i^k-y_i^{k}$, we have the following parallel DR algorithm
\begin{equation}\label{eqDR3}
\begin{cases}\hat{x}_{i}^{k} & = x_{i}^{k}+w_{i}^{k}, \quad \forall i \in[n] \\
\tilde{x}^{k} & =\frac{1}{n} \sum_{i=1}^{n} \hat{x}_{i}^{k},\quad \forall i \in[n] \\ \bar{x}^{k} & =\operatorname{prox}_{\eta g}\left(\tilde{x}^{k}\right), \\ 
 x_{i}^{k+1} & =\operatorname{prox}_{\eta f_{i}}\left(\bar{x}^{k}-w_i^k\right), \quad \forall i \in[n]. \\ 
 w_{i}^{k+1} & =w_{i}^{k}+x_{i}^{k+1}-\bar{x}^{k}, \quad \forall i \in[n] \end{cases}
\end{equation}
\begin{remark}
Note that \eqref{eq:parrallel},\eqref{eqDR2} and \eqref{eqDR3} are essentially the same parallel algorithm under a change of execution order and variables.
\end{remark}

\subsection{FedDR-II}
From section~\ref{section_DRS}, we observe that the only difference between full parallel DRS and FedDR  is that FedDR only requires a subset of clients to update their variables, while full parallel DRS requires full participation.
Similarly, by only considering partial participation in \eqref{eqDR3}, we introduce the intermediate FedDR-II algorithm. We now describe each step of a single epoch of FedDR-II:
\begin{enumerate}
\item \textbf{Initialization:} Given an initial vector $x^{0} \in \operatorname{dom}(F)$ and tolerances $\epsilon_{i, 0} \geq 0$.
Initialize the server with $\bar{x}^{0}=x^{0}$.
Initialize all users $i \in[n]$ with $w_{i}^{0}=0$ and $x_{i}^{0}= x^{0}$.
\item \textbf{The $k$-th iteration: $(k \geq 0)$} Sample a proper subset $\mathcal{S}_{k} \subseteq[n]$ so that $\mathcal{S}_{k}$ represents the subset of active clients.
\item \textbf{Client update (Local):} For each client $i \in \mathcal{S}_{k}$,  update 
$
\hat{x}_{i}^{k}= x_{i}^{k}+w_{i}^{k} .
$
Clients $i \notin \mathcal{S}_{k}$ do nothing, i.e.
$$
\left\{\begin{array}{lll}
\hat{x}_{i}^{k} & = & \hat{x}_{i}^{k-1} \\
x_{i}^{k} & = & x_{i}^{k-1} \\
w_{i}^{k} & = & w_{i}^{k-1}
\end{array}\right.
$$
\item \textbf{Communication:} Each user $i \in \mathcal{S}_{k}$ sends only $\hat{x}_{i}^{k}$ to the server.
\item \textbf{Server update:} The server aggregates $\tilde{x}^{k}=\frac{1}{n} \sum_{i=1}^{n} \hat{x}_{i}^{k}$, and then compute $\bar{x}^{k}=\operatorname{prox}_{\eta g}\left(\tilde{x}^{k}\right)$.
\item \textbf{Communication (Broadcast):} Each user $i \in \mathcal{S}_{k}$ receives $\bar{x}^k$ from the server.
\item \textbf{Client update (Local):} For each user $i \in \mathcal{S}_{k}$, given $\epsilon_{i, k+1} \geq 0$, it updates
$$
\left\{\begin{array}{l}
x_{i}^{k+1} \approx \operatorname{prox}_{\eta f_{i}}\left(\bar{x}^k-w_{i}^{k}\right) \\
w_{i}^{k+1}=w_{i}^{k}+x_{i}^{k+1}-\bar{x}^{k}.
\end{array}\right.
$$
Each user $i \notin \mathcal{S}_{k}$ does nothing, i.e.
$$
\left\{\begin{array}{lll}
w_{i}^{k+1} & = & w_{i}^{k} \\
x_{i}^{k+1} & = & x_{i}^{k}
\end{array}\right.
$$
\end{enumerate}
\begin{remark}
FedDR and FedDR-II are equivalent because they are partial participation version of \eqref{eq:parrallel} and \eqref{eqDR3} respectively.
\end{remark}
\subsection{Solving the dual problem using FedDR-II}\label{secFedDR_convert}
In this subsection, we use FedDR-II to solve the dual problem~\eqref{eqdual2}, introducing a new algorithm called FedADMM. We call this algorithm FedADMM because it is derived from applying FedDR-II to the dual problem~\eqref{eqdual2}. Let us define the augmented Lagrangian functions associated to~\eqref{eqform2} as
\begin{equation}\label{eqlag}
\mathcal{L}_i(x_i,\bar{x}^k,z_i)=f_{i}\left({x}_{i}\right)
 + g(\bar{x}^{k})+\left\langle z_{i}^k, {x}_{i}-\bar{x}^{k}\right\rangle+\frac{\eta}{2 }\left\|x_{i}-\bar{x}^{k}\right\|^{2}
 \end{equation} 
where $\eta$ denotes penalty parameter. Finally, we define $\Delta \hat{x}_{i}^{k}=\hat{x}_{i}^{k+1}-\hat{x}_{i}^{k}.$ With everything defined,  FedADMM  is shown in Algorithm~\ref{algADMM}.
\begin{algorithm}
\caption{Federated ADMM Algorithm (FedADMM)} \label{algADMM}
\begin{algorithmic}[1]
\State \textbf{Initialize}  $x^{0}, \eta>0, K,
\text{ and tolerances } \epsilon_{i, 0} (i \in[n]).$\\
\textbf{Initialize} the server with $\bar{x}^{0}=x^{0}$\\
\textbf{Initialize} all clients  with $z_{i}^{0}=0$ and $x_{i}^{0}=\hat{x}_i^0= x^{0}.$
\State \textbf{for} $k=0, \ldots, K$ \textbf{ do } \
\State \quad $\text {Randomly sample }\mathcal{S}_k \subseteq[n] \text{ with size } S$.
\State \quad $\rhd${ Client side}
\State \quad \textbf{for each client} $i \in \mathcal{S}_{k}$ \textbf{do} 
\State \quad \quad receive $\bar{x}^k$ from the server.
\State \quad \quad${x}_{i}^{k+1}\approx \underset{x_i}{\arg\min} \ \mathcal{L}_{i}\left({x}_{i}, \bar{x}^{k}, z_{i}^{k}\right)$ \
\State \quad \quad $z_{i}^{k+1}=z_{i}^{k}+\eta \left({x}_{i}^{k+1}-\bar{x}^k\right) \quad \diamondsuit \text {Dual updates}$ \
\State\quad\quad $\hat{x}_{i}^{k+1}=x_{i}^{k+1}+\frac{1}{\eta} z_{i}^{k+1}$ \
\State \quad \quad send  $\Delta \hat{x}_{i}^{k}=\hat{x}_{i}^{k+1}-\hat{x}_{i}^{k} \text { back to the server }$\
\State \quad \textbf{end for}
\State \quad $\rhd${ Server side}
\State \quad aggregation $\tilde{x}^{k+1}=\tilde{x}^{k}+\frac{1}{n} \sum_{i \in \mathcal{S}_{k}} \Delta \hat{x}_{i}^{k}$\
\State \quad update $\bar{x}^{k+1}=\operatorname{prox}_{g/\eta }\left(\tilde{x}^{k+1}\right)$
\State\textbf{end for}
\end{algorithmic}
\end{algorithm}

When $g \equiv 0,$ the server-side steps 15-16 of FedADMM reduce to the single step: $$\bar{x}^{k+1}=\tilde{x}^{k+1}=\tilde{x}^{k}+\frac{1}{n} \sum_{i \in \mathcal{S}_{k}} \Delta \hat{x}_{i}^{k} =\frac{1}{n}\sum_{i=1}^n \hat{x}_i^{k+1}.$$ In this case, the updating rules of FedADMM are essentially the same as FedPD in~\cite{zhang2021fedpd}. Both  compute the local model $x_i^{k+1}$ by first minimizing~\eqref{eqlag}, followed by updating the dual variable  $\lambda_i^{k+1}$, and then aggregating  $\hat{x}_i^{k+1}$ to achieve the global model $\bar{x}^{k+1}$. However, FedADMM allows for partial participation (only chooses a subset of clients to update) while FedPD requires all clients to update at each communication rounds, making it less practical and applicable in real world scenarios. 

Note that FedADMM can handle the case where $g \not\equiv 0$ whereas FedPD didn't consider this more general formulation. Just like step 11 (approximately evaluating $\operatorname{prox}_{\eta f_{i}}$) in FedDR, FedADMM obtains the new local model $x_i^{k+1}$ by inexactly solving~\eqref{eqlag}. Note that we do not specify how to (approximately) solve the proximal steps or Langrangian minimization step in (either) algorithm. Various oracles are specified   in ~\cite{zhang2021fedpd}.

\section{Theoretical Analysis}\label{thm}
We now present the main theoretical results of the paper. Namely, an equivalence between FedDR and FedADMM. Based on this, we leverage the FedDR convergence results~\cite{tran2021feddr} to show that FedADMM converges under partial participation.

We say that two iterative optimization algorithms are ``equivalent'' if  they produce sequences $(x^k)_{k\ge 0}$ and $(y^k)_{k\ge 0}$ such that there exists a unique linear mapping between the two sequences. More general equivalence classes are defined and studied in~\cite{ZhaLU21}.
\begin{thm}
\textbf{(Equivalence between FedDR and FedADMM)} Let $(x_i^k,z_i^k,\bar{x}^k)_{k\ge0}$ be a sequence generated by FedADMM with penalty parameter $\eta$, and $(s_i^k,u_i^k,\hat{u}_i^k,\bar{v}^k)$ a sequence generated by FedDR with parameter $\frac{1}{\eta}$. Then FedADMM and FedDR are equivalent.
\end{thm}
\begin{proof}
For each triplet $(x_i^k,z_i^k,\bar{x}^k)$ at the $k$-th iteration of FedADMM with stepsize $\eta$, define
\begin{equation*}
\begin{cases}
s_i^k & = x_i^k - z_i^k / \eta  \\
 u_i^k & = x_i^k  \\
\hat{u}_i^k&= x_i^k +z_i^k /\eta \\
\bar{v}^k & =  \bar{x} ^k
\end{cases}\quad \text{and }
\begin{cases}
s_i^{k+1} & = x_i^{k+1} - z_i^{k+1} / \eta  \\
 u_i^{k+1} & = x_i^{k+1}  \\
\hat{u}_i^{k+1}&= x_i^{k+1} +z_i^{k+1} /\eta \\
\bar{v}^{k+1} & =  \bar{x} ^{k+1}
\end{cases}
\end{equation*}
Then $(s_i^k,u_i^k,\hat{u}_i^k,\bar{v}^k)$ and $(s_i^{k+1},u_i^{k+1},\hat{u}_i^{k+1},\bar{v}^{k+1})$ satisfy the updating rule of FedDR
\begin{equation*}
\begin{cases}s_{i}^{k+1} & =s_{i}^{k} +(\bar{v}^k -u_i^k), \quad \forall i \in \mathcal{S}_k, \\
u_{i}^{k+1} & =\operatorname{prox}_{r f_{i}}\left(s_i^{k+1}\right), \quad \forall i \in \mathcal{S}_k,\\
\hat{u}_{i}^{k+1} & =2u_{i}^{k+1} -s_{i}^{k+1}, \quad \forall i \in  \mathcal{S}_k,\\
\bar{v}^{k+1} & = \operatorname{prox}_{r g}(\frac{1}{n}\sum_{i=1}^n \hat{u}_{i}^{k+1} ),
\end{cases}
\end{equation*}
where $r=1/\eta$ and when $i \notin \mathcal{S}_k$
\begin{equation*}
\begin{cases}
s_{i}^{k+1} & =  s_{i}^{k}, \\
u_{i}^{k+1} & = u_{i}^{k}, \\
\hat{u}_{i}^{k+1} & = \hat{u}_{i}^k
\end{cases}
\end{equation*}
where the same sampling realizations $\mathcal{S}_k$ are used at each iteration for both algorithm.

We have
\begin{align*}
s_{i}^{k} +(\bar{v}^k -u_i^k)&=x_i^k - z_i^k / \eta +(\bar{x}^k-x_i^k)\\ &=x_i^{k+1}-z_i^{k}/\eta +\bar{x}^k -x_{i}^{k+1}\\&\stackrel{(a)}
 =x_i^{k+1}-z_i^{k+1}/\eta=s_i^{k+1}
\end{align*}
where (a) is due to the dual updates (line 10) in FedADMM algorithm. Moreover,
\begin{align*}
u_i^{k+1} = x_i^{k+1} &= \underset{x_i}{\arg\min} \ \mathcal{L}_{i}\left({x}_{i}, \bar{x}^{k}, z_{i}^{k}\right)\\&=\operatorname{prox}_{r f_{i}}(\bar{x}^k-z_i^k/\eta)
\\&\stackrel{(b)}=\operatorname{prox}_{r f_{i}}(s_i^{k+1})
\end{align*}
where (b) uses the fact that $\bar{x}^k-z_i^k/\eta=s_{i}^{k} +(\bar{v}^k -u_i^k)=s_{i}^{k+1}$. 

Finally, note that 
\[\hat{u}_i^{k+1}=2u_{i}^{k+1} -s_{i}^{k+1}=x_i^{k+1} +z_i^{k+1} /\eta,\] 
which gives
\begin{equation}
    \begin{aligned}
    \bar{v}^{k+1}=\bar{x}^{k+1}&\stackrel{(c)}=\operatorname{prox}_{r g}\left(\sum_{i=1}^n \left(x_{i}^{k+1}+\frac{1}{\eta}z_i^{k+1} \right)\right)=\operatorname{prox}_{r g}(\frac{1}{n}\sum_{i=1}^n \hat{u}_{i}^{k+1} )
    \end{aligned}
\end{equation}
where (c) comes from the FedADMM updating rule (line 11-16 in Alg~\ref{algADMM}).
\end{proof}

Since we have proved the equivalence of FedDR and FedADMM for arbitrary (nonconvex) problems,   FedADMM will directly inherit the convergence properties  of FedDR, specifically at rate  $\mathcal O(\frac{1}{k})$. The explicit convergence rate of FedADMM is characterized in the following theorem which is a direct application of Theorem 3.1 in~\cite{tran2021feddr}.
\begin{thm}
Suppose that Assumptions 1, 2, and 3 hold and  $\gamma_1,\gamma_2,\gamma_3,\gamma_4 >0$ are constants. Let $\left(x_{i}^{k}, z_{i}^{k}, \hat{x}_{i}^{k}, \bar{x}^{k}\right)_{k\ge0}$ be generated by Alg \ref{algADMM} (FedADMM) using penalty parameter $\eta$ that satisfies  $$
\eta>\frac{4 L\left(1+2  \gamma_{4}\right)}{\sqrt{9-16 \gamma_{4}\left(1+4 \gamma_{4}\right)}-1}.
$$Then when $g\equiv 0$, the following holds
\begin{equation*}
\frac{1}{K+1} \sum_{k=0}^{K} \mathbb{E}\left[\left\|\nabla f\left(\bar{x}^{k}\right)\right\|^{2}\right] \leq \frac{C_{1}\left[F\left(x^{0}\right)-F^{\star}\right]}{K+1}+\frac{1}{n(K+1)} \sum_{k=0}^{K} \sum_{i=1}^{n}\left(C_{2} \epsilon_{i, k}^{2}+C_{3} \epsilon_{i, k+1}^{2}\right)
\end{equation*}
where $\hat{\eta} =1/\eta,$ $\beta, \rho_{1}$, and $\rho_{2}$ are  defined as 
\begin{equation*}
\begin{cases}
\beta &=\frac{\hat{\mathbf{p}} \left[2-(L \hat{\eta}+1)-2 L^{2} \hat{\eta}^{2}-4 \gamma_{4} \left(1+L^{2} \hat{\eta}^{2}\right)\right]}{2 \hat{\eta}\left(1+\gamma_{1}\right)\left(1+L^{2} \hat{\eta}^{2}\right)}>0 \\
\rho_{2} &=\frac{2(1+\hat{\eta} L)^{2}}{\gamma_{4} \hat{\eta} }+\frac{\left(1+\hat{\eta}^{2} L^{2}\right)}{\hat{\eta}}
\\
&+\frac{\left[2-(L \hat{\eta}+1)-2 L^{2} \hat{\eta}^{2}-4  \gamma_{4}\left(1+L^{2} \hat{\eta}^{2}\right)\right]}{2 \hat{\eta}\left(1+L^{2} \hat{\eta}^{2}\right) \gamma_{1}} \\
\rho_{1} &=\rho_{2}+\frac{\left(1+\hat{\eta}^{2} L^{2}\right)}{\hat{\eta}}
\end{cases}
\end{equation*} and the constants are
\begin{equation*}
C_{1}=\frac{2(1+\hat{\eta} L)^{2}\left(1+\gamma_{2}\right)}{\hat{\eta}^{2} \beta}, \ C_{2}=\rho_{1} C_{1}, \
C_{3}=\rho_{2} C_{1}+\frac{(1+\hat{\eta} L)^{2}\left(1+\gamma_{2}\right)}{\hat{\eta}^{2} \gamma_{2}} .
\end{equation*}
and $
\hat{p}=\min \left\{p_{i}: i \in[n]\right\}>0
$ in Assumption~\ref{sampling}.
\end{thm}

\begin{coro}
If the accuracy sequence $\epsilon_{i, k}$ (for all $i \in [n]$ and $k>0$) at Step 8 in Alg~\ref{algADMM} satisfies $\frac{1}{n} \sum_{i=1}^{n} \sum_{k=0}^{K+1} \epsilon_{i, k}^{2} \leq D$ for a given constant $D>0$ and all $K \geq 0$. Then, FedADMM needs
$$
K=\left\lfloor\frac{C_{1}\left[F\left(x^{0}\right)-F^{\star}\right]+\left(C_{2}+C_{3}\right) D}{\varepsilon^{2}}\right\rfloor \equiv \mathcal{O}\left(\varepsilon^{-2}\right)
$$
iterations to achieve $\frac{1}{K+1} \sum_{k=0}^{K} \mathbb{E}\left[\left\|\nabla f\left(\tilde{x}^{k}\right)\right\|^{2}\right] \le \varepsilon^2,$ where $\tilde{x}^{K}$ is randomly selected from $\{\bar{x}^{0}, \bar{x}^{1},\cdots, \bar{x}^{K}\}$. In other words, after $K=\mathcal{O}(\varepsilon^{-2})$ iterations, $
\tilde{x}^{K}$ is an $\varepsilon$-stationary solution of problem~\eqref{eqobjective} when $g\equiv 0$.
\end{coro}
\begin{remark}
Our convergence analysis can be easily extended to $g\not\equiv 0$, as long as we change the suboptimal condition into the gradient mapping as in~\cite{tran2021feddr}. To make $\frac{1}{n} \sum_{i=1}^{n} \sum_{k=0}^{K+1} \epsilon_{i, k}^{2} \leq D$ hold, interested readers could refer to Remark 3.1 in~\cite{tran2021feddr}.
\end{remark}
\begin{remark}
Although FedADMM is a partial participation version of FedPD when $g\equiv 0$, its communication complexity is still $\mathcal{O}(\varepsilon^{-2})$, which matches the lower bound (up to constant factors) in~\cite{zhang2021fedpd}.
\end{remark}

\section{Numerical Simulations}
To demonstrate the equivalence of FedDR and FedADMM, we conduct diverse simulations on both synthetic and real datasets. It is worthwhile to mention that our goal is to show the equivalence of the algorithms, \emph{not} to compare their performance with other algorithms. Performance profiling of FedPD and FedDR can be found in~\cite{tran2021feddr,zhang2021fedpd}. We have not attempted to optimize any hyperparameters. All the experiments run on Google Colab with default CPU setup.

\noindent\textbf{Datasets:} We first generate synthetic non-iid datasets by following the same setup as in~\cite{shamir2014communication} and denote them as \texttt{synthetic-$(\alpha, \beta)$}. Here $\alpha$ controls how much local models differ from each other and $\beta$ controls how much the local data at each device differs from that of other devices. We run the experiments by using the unbalanced datasets: \texttt{synthetic-(0, 0)}, \texttt{synthetic-(0.5, 0.5)} and \texttt{synthetic-(1, 1)}. We then compare FedADMM with FedDR on the FEMNIST data set~\cite{caldas2018leaf}. FEMNIST is a more complex 62-class Federated Extended MNIST dataset. It consists of handwritten characters including: numbers 1-10, 26 upper-and lower-case letters A-Z and a-z from different writers and is also separated by the writers, therefore the
dataset is non-iid. 

\noindent\textbf{Models and Hyper-parameters:} For all the synthetic datasets, we use the  model described in~\cite{tran2021feddr}: a neutral network with a single hidden layer. The network architecture is  $60\times32\times 10$ corresponding to \emph{input layer}$ \times$ \emph{hidden layer} $\times$ \emph{output layer} size. For FEMNIST data, we use the same model as~\cite{caldas2018leaf}, which consists of 2 convolutional layers and two fully
connected layers, with 62 neurons in the output layer matching
the number of classes in the FEMNIST dataset. For all the experiments, we use $\eta=1$ and $\alpha=1$. As in~\cite{zhang2021fedpd}, we choose stochastic gradient descent as a local solver with 300 local iterations to solve the step 11 in FedDR and the step 9 in FedADMM. The mini-batch size in calculating the stochastic gradient is 2 and the learning rate is 0.01. We stress that we do not attempt to optimize these parameters.

\noindent\textbf{Implementation:} We use the uniform sampling scheme to select the clients in each round. The total number of clients is 30 and we set the number of active clients in each round as 10. To provide a fair comparison, we use the same random seeds across all algorithms.

After running multiple experiments on different datasets and models, from figure~\ref{fig:femnist} and \ref{fig:synthetic}, we could observe that the training accuracy and loss of FedDR and FedADMM coincide at each iteration, which verifies our theoretical analysis in section~\ref{thm}. 

\begin{figure}[t!]
     \centering
     \begin{subfigure}[t]{0.48\textwidth}
         \centering
         \includegraphics[width=\textwidth]{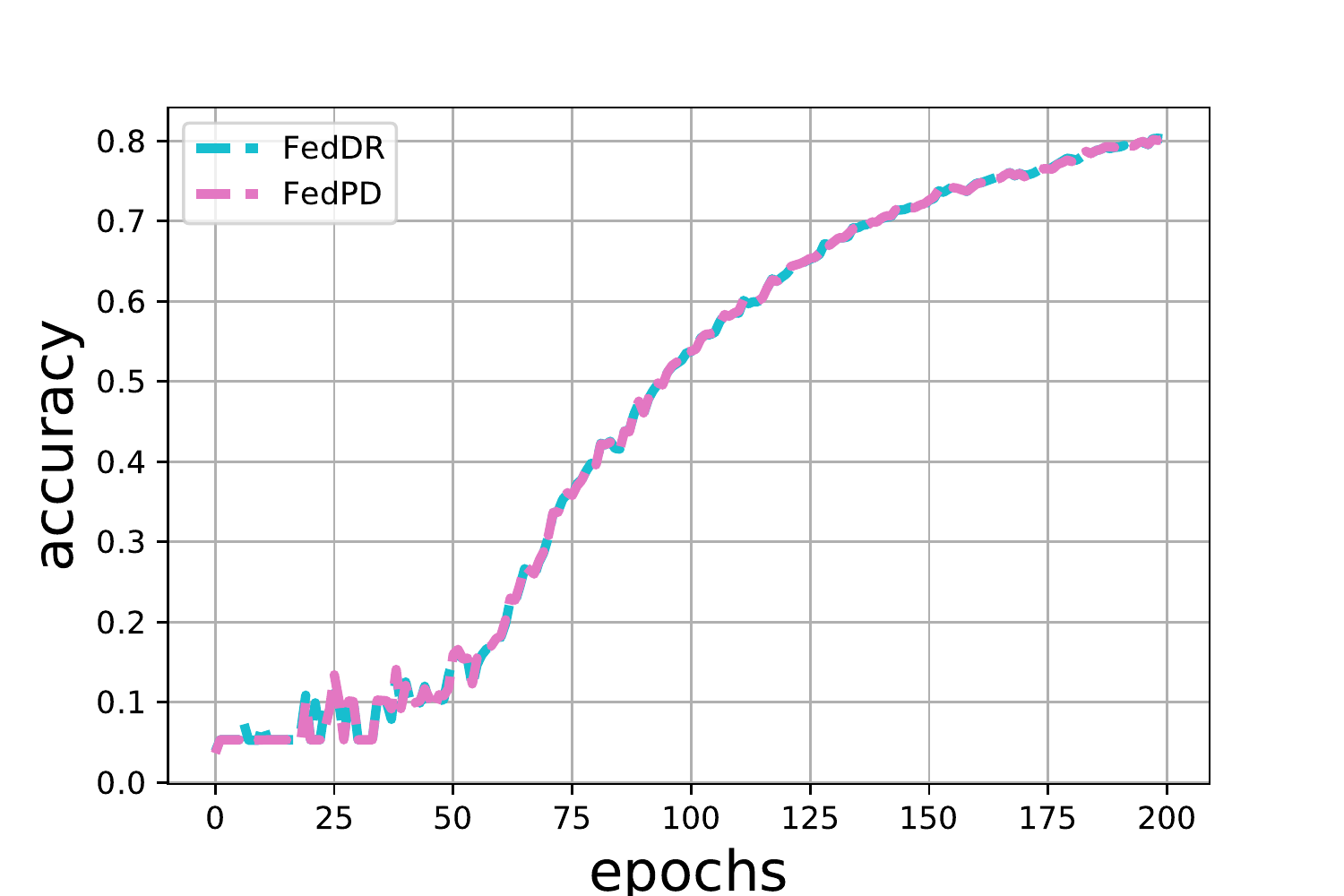}
         \caption{ FEMNIST accuracy }
         \label{fig:syn01}
     \end{subfigure}
     \hfil 
     \begin{subfigure}[t]{0.48\textwidth}
         \centering
         \includegraphics[width=\textwidth]{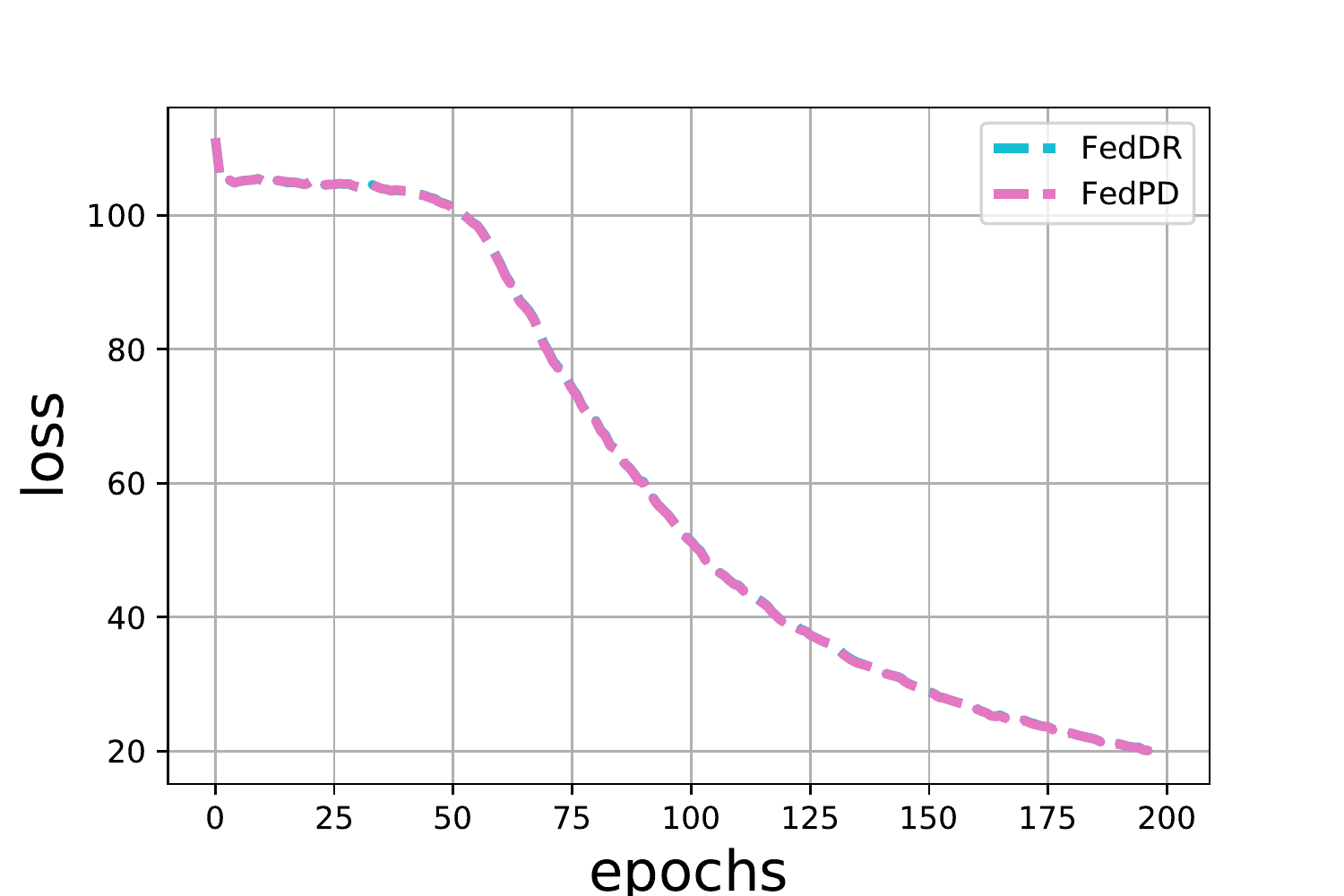}
          \caption{FEMNIST loss }
        \label{fig:syn02}  
     \end{subfigure}
     \caption{Identical performance of FedDR and FedADMM in terms of training accuracy and cross-entropy training loss of FEMNIST dataset}
     \label{fig:femnist}
\end{figure}



\begin{figure*}[htb] 
\centering 
\begin{subfigure}{0.3\textwidth}
  \includegraphics[width=\linewidth]{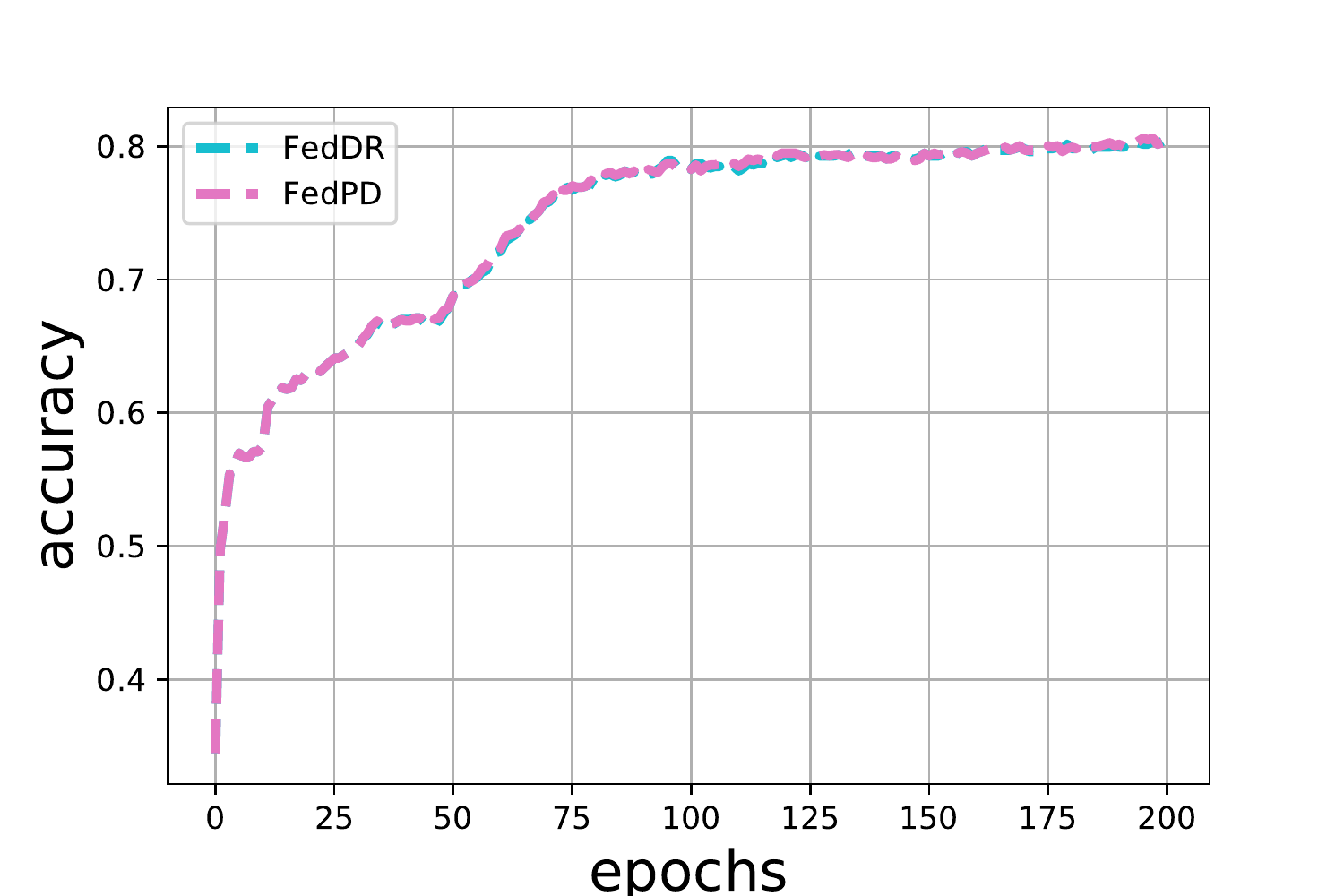}
   \caption{\texttt{synthetic(0,0)}  accuracy}
  \label{fig:1}
\end{subfigure}\hfil 
\begin{subfigure}{0.3\textwidth}
  \includegraphics[width= \linewidth]{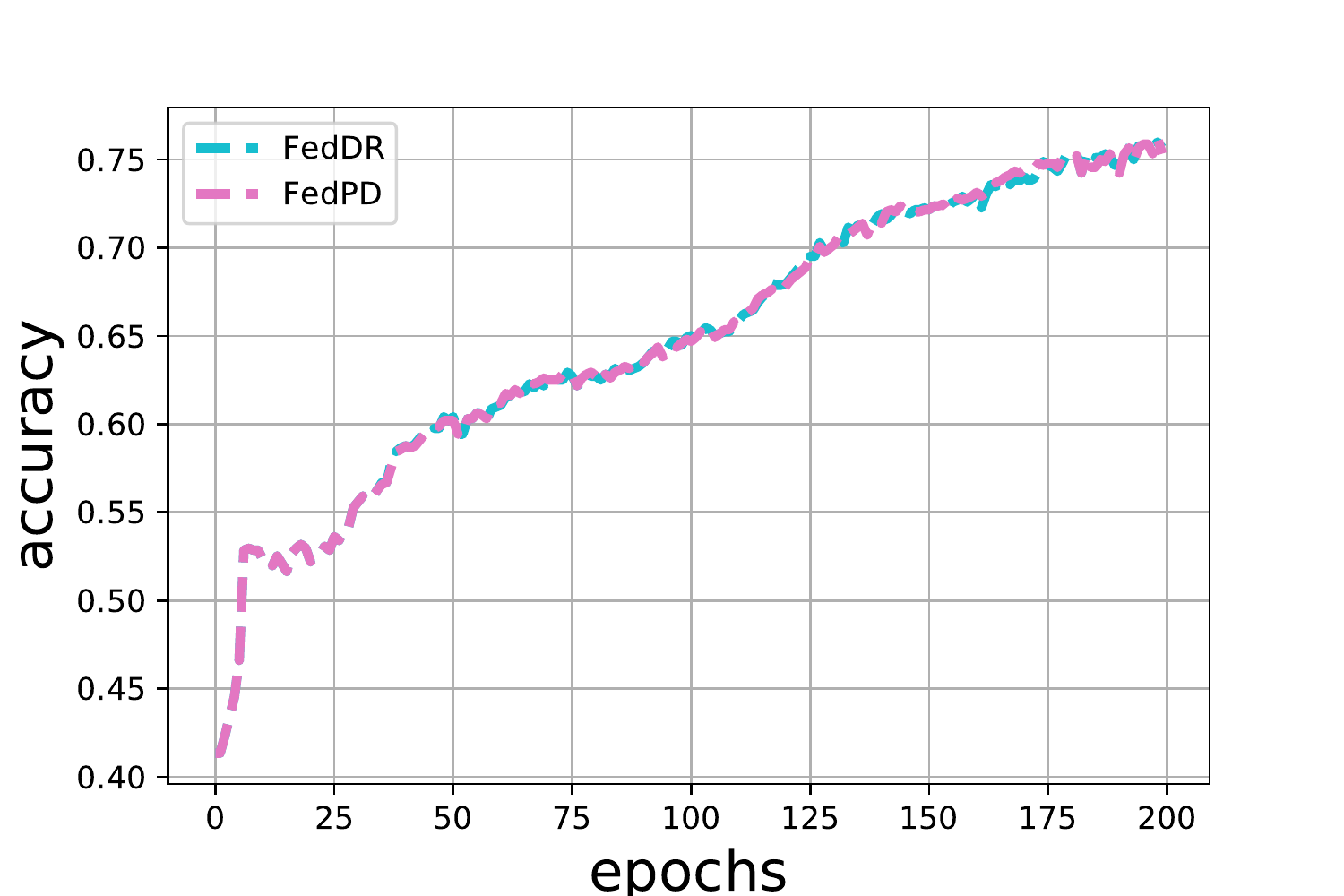}
 \caption{\texttt{synthetic(0.5, 0.5)} accuracy}
  \label{fig:2}
\end{subfigure}\hfil 
\begin{subfigure}{0.3\textwidth}
  \includegraphics[width=\linewidth]{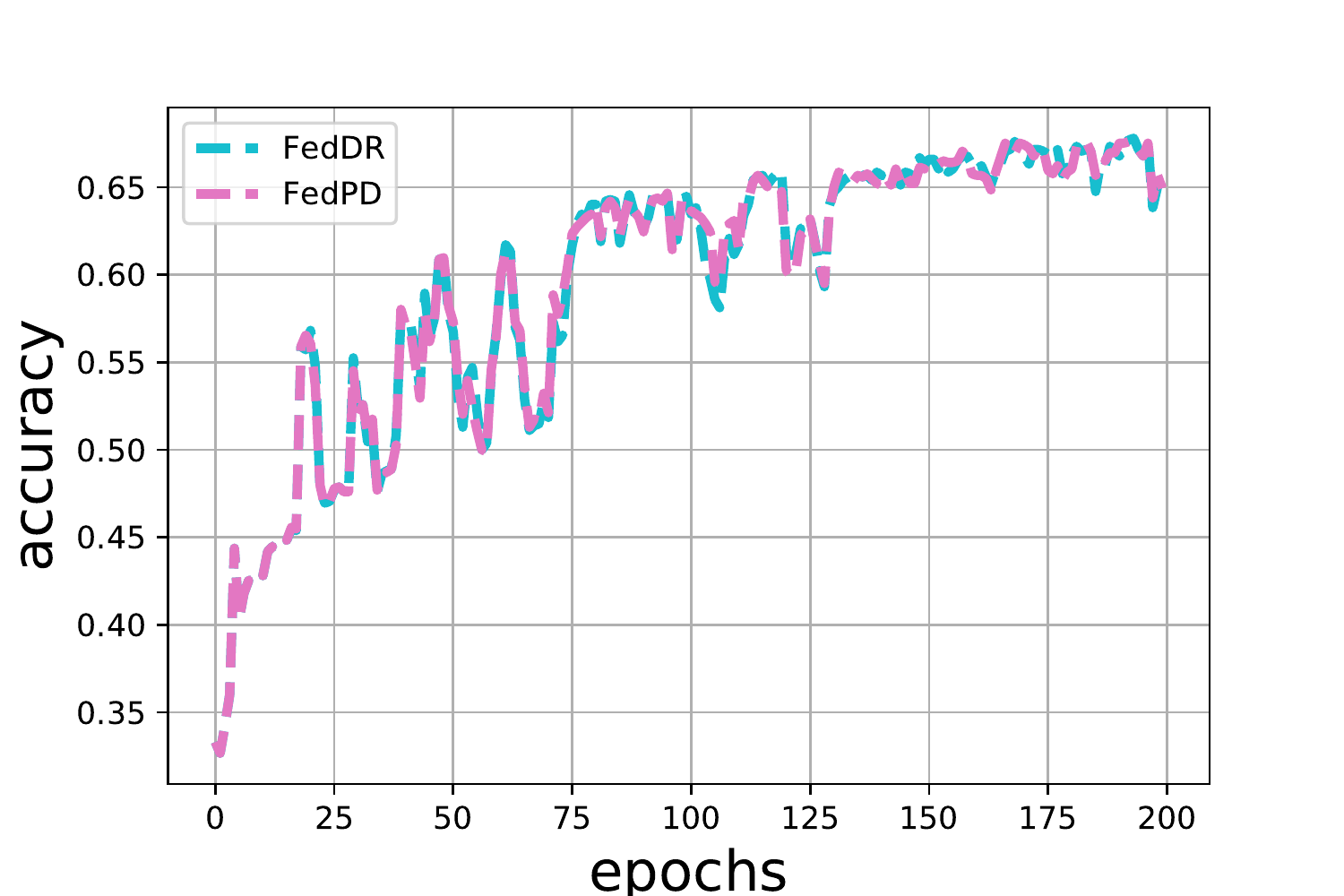}
 \caption{\texttt{synthetic(1,1)} accuracy}
  \label{fig:3} 
\end{subfigure}

\medskip
\begin{subfigure}{0.3\textwidth}
  \includegraphics[width=\linewidth]{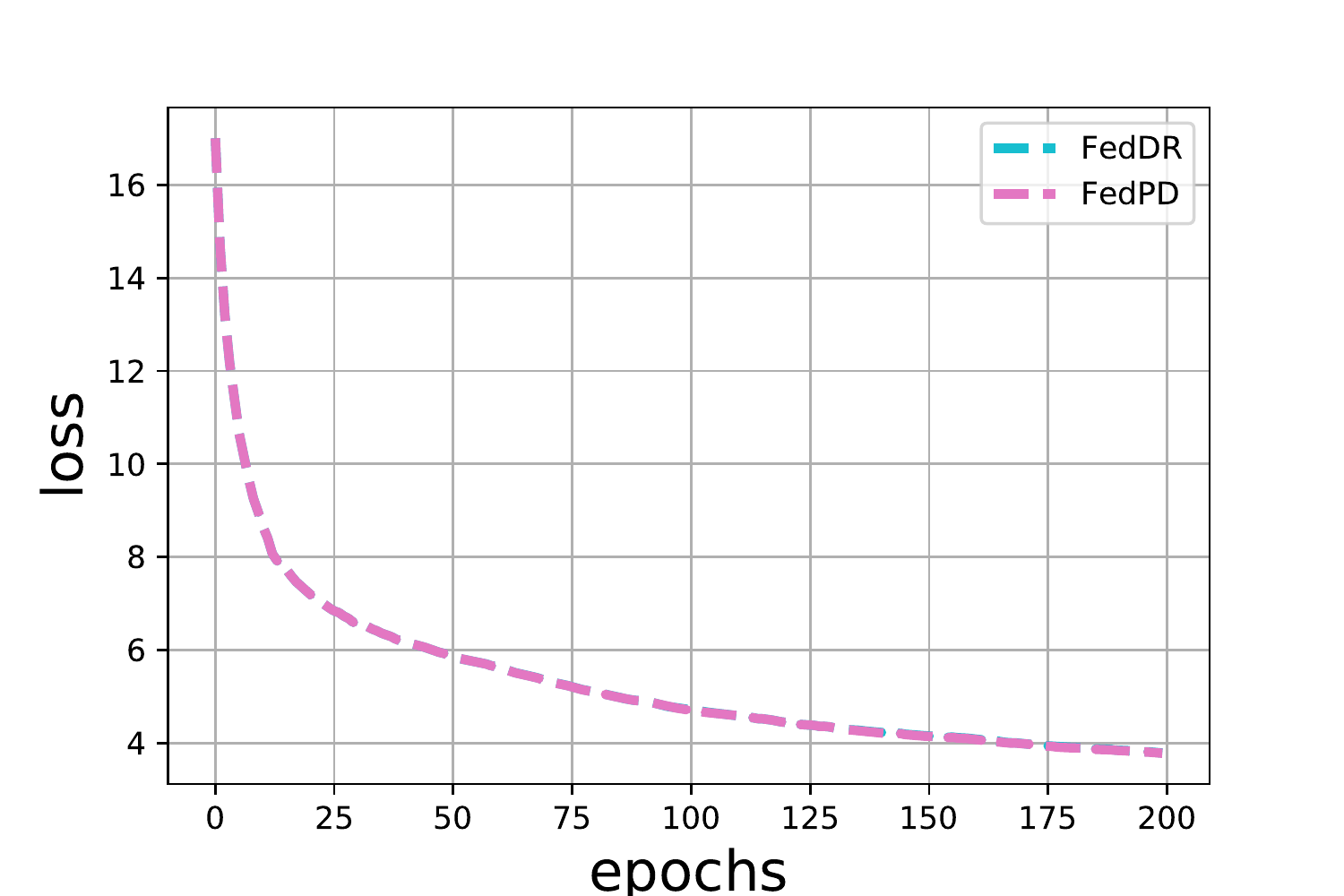}
 \caption{\texttt{synthetic(0,0)}  loss}
  \label{fig:4}
\end{subfigure}\hfil 
\begin{subfigure}{0.3\textwidth}
  \includegraphics[width=\linewidth]{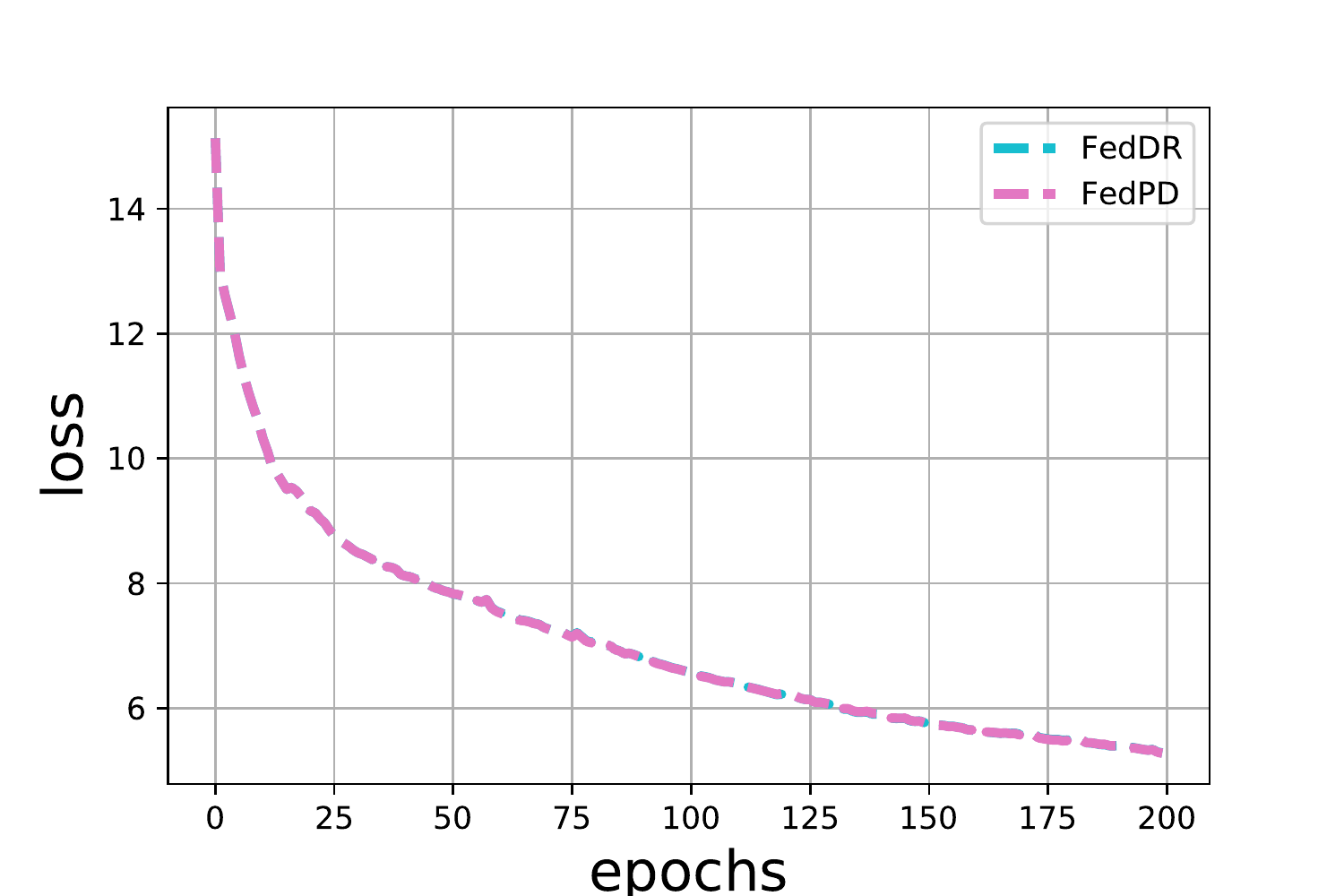}
 \caption{\texttt{synthetic(0.5,0.5)}  loss}
  \label{fig:5}
\end{subfigure}\hfil 
\begin{subfigure}{0.3\textwidth}
  \includegraphics[width=\linewidth]{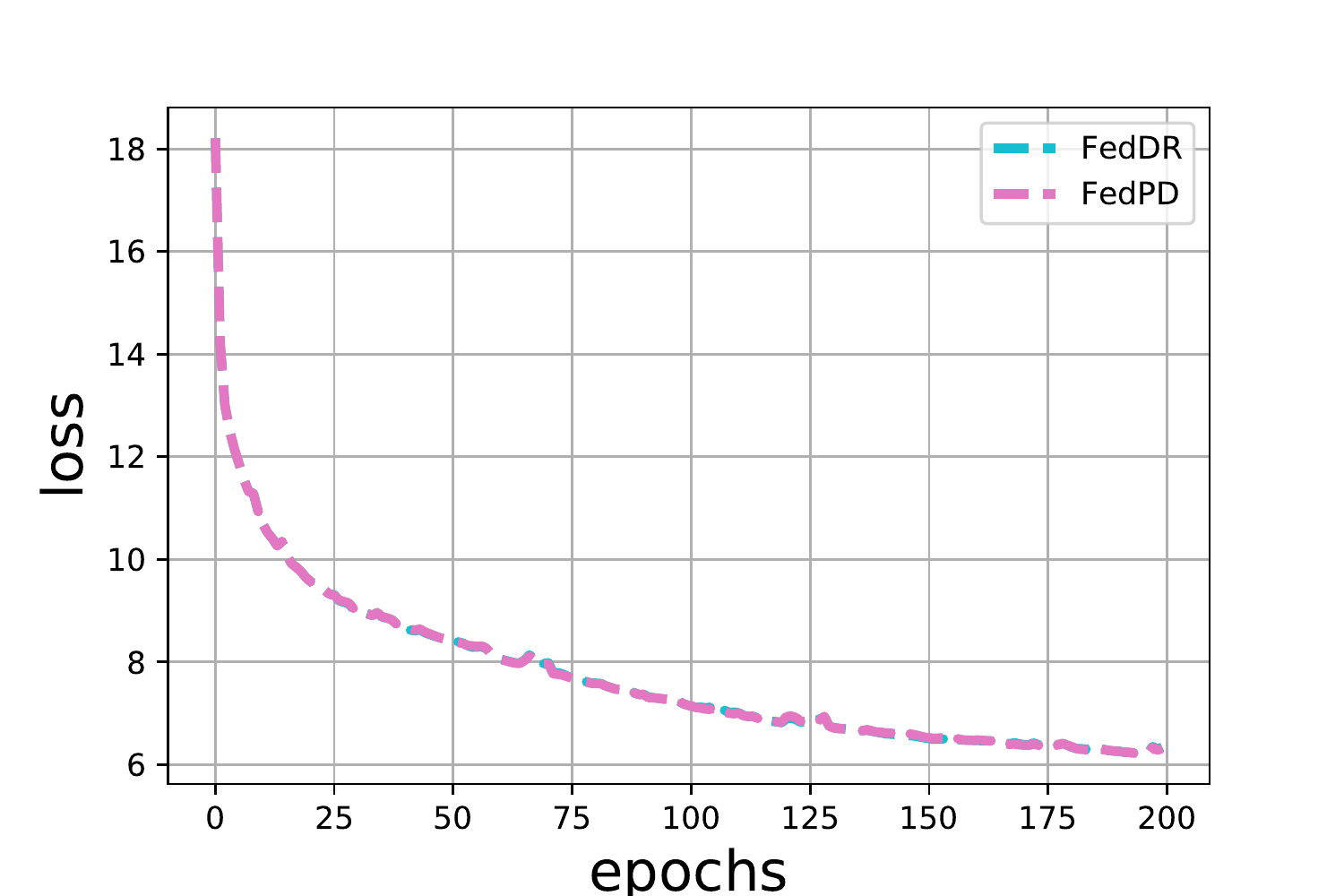}
   \caption{\texttt{synthetic(1,1)}  loss} 
  \label{fig:6} 
\end{subfigure}
\caption{Identical performance of FedDR and FedADMM in terms of training accuracy and cross-entropy training loss of synthetic datasets}
\label{fig:synthetic}
\end{figure*}

\section{Conclusion}
We have developed a new federated learning algorithm, FedADMM, for finding stationary points in non-convex composite optimization problems. Current work is focused on incorporating convex constraints into the algorithm, proposing an asynchronous algorithm, asyncFedADMM, and applying it to non-localizable model predictive control problems where communication efficiency is necessary~\cite{AloSAM21}.
\bibliographystyle{IEEEtranS}
\bibliography{reference}

\begin{thebibliography}{10}
\providecommand{\url}[1]{#1}
\csname url@samestyle\endcsname
\providecommand{\newblock}{\relax}
\providecommand{\bibinfo}[2]{#2}
\providecommand{\BIBentrySTDinterwordspacing}{\spaceskip=0pt\relax}
\providecommand{\BIBentryALTinterwordstretchfactor}{4}
\providecommand{\BIBentryALTinterwordspacing}{\spaceskip=\fontdimen2\font plus
\BIBentryALTinterwordstretchfactor\fontdimen3\font minus
  \fontdimen4\font\relax}
\providecommand{\BIBforeignlanguage}[2]{{%
\expandafter\ifx\csname l@#1\endcsname\relax
\typeout{** WARNING: IEEEtranS.bst: No hyphenation pattern has been}%
\typeout{** loaded for the language `#1'. Using the pattern for}%
\typeout{** the default language instead.}%
\else
\language=\csname l@#1\endcsname
\fi
#2}}
\providecommand{\BIBdecl}{\relax}
\BIBdecl

\bibitem{acar2021federated}
D.~A.~E. Acar, Y.~Zhao, R.~M. Navarro, M.~Mattina, P.~N. Whatmough, and
  V.~Saligrama, ``Federated learning based on dynamic regularization,''
  \emph{arXiv preprint arXiv:2111.04263}, 2021.

\bibitem{AloSAM21}
C.~A. Alonso, J.~Shuang, J.~Anderson, and N.~Matni, ``Distributed and localized
  model predictive control. part i: Synthesis and implementation,'' \emph{arXiv
  preprint arXiv:2110.07010}, 2021.

\bibitem{boyd2011distributed}
S.~Boyd, N.~Parikh, E.~Chu, B.~Peleato, J.~Eckstein \emph{et~al.},
  ``Distributed optimization and statistical learning via the alternating
  direction method of multipliers,'' \emph{Foundations and
  Trends{\textregistered} in Machine learning}, vol.~3, no.~1, pp. 1--122,
  2011.

\bibitem{caldas2018leaf}
S.~Caldas, S.~M.~K. Duddu, P.~Wu, T.~Li, J.~Kone{\v{c}}n{\`y}, H.~B. McMahan,
  V.~Smith, and A.~Talwalkar, ``{Leaf: A benchmark for federated settings},''
  \emph{arXiv preprint arXiv:1812.01097}, 2018.

\bibitem{douglas1956numerical}
J.~Douglas and H.~H. Rachford, ``On the numerical solution of heat conduction
  problems in two and three space variables,'' \emph{Transactions of the
  American mathematical Society}, vol.~82, no.~2, pp. 421--439, 1956.

\bibitem{eckstein1989splitting}
J.~Eckstein, ``Splitting methods for monotone operators with applications to
  parallel optimization,'' Ph.D. dissertation, Massachusetts Institute of
  Technology, 1989.

\bibitem{fukushima1992application}
M.~Fukushima, ``Application of the alternating direction method of multipliers
  to separable convex programming problems,'' \emph{Computational Optimization
  and Applications}, vol.~1, no.~1, pp. 93--111, 1992.

\bibitem{gabay1983chapter}
D.~Gabay, ``{Chapter ix applications of the method of multipliers to
  variational inequalities},'' in \emph{Studies in mathematics and its
  applications}.\hskip 1em plus 0.5em minus 0.4em\relax Elsevier, 1983,
  vol.~15, pp. 299--331.

\bibitem{gabay1976dual}
D.~Gabay and B.~Mercier, ``A dual algorithm for the solution of nonlinear
  variational problems via finite element approximation,'' \emph{Computers \&
  mathematics with applications}, vol.~2, no.~1, pp. 17--40, 1976.

\bibitem{giselsson2016linear}
P.~Giselsson and S.~Boyd, ``{Linear convergence and metric selection for
  Douglas-Rachford splitting and ADMM},'' \emph{IEEE Transactions on Automatic
  Control}, vol.~62, no.~2, pp. 532--544, 2016.

\bibitem{kairouz2021advances}
P.~Kairouz, H.~B. McMahan, B.~Avent, A.~Bellet, M.~Bennis, A.~N. Bhagoji,
  K.~Bonawitz, Z.~Charles, G.~Cormode, R.~Cummings \emph{et~al.}, ``Advances
  and open problems in federated learning,'' \emph{Foundations and
  Trends{\textregistered} in Machine Learning}, vol.~14, no. 1--2, pp. 1--210,
  2021.

\bibitem{karimireddy2020scaffold}
S.~P. Karimireddy, S.~Kale, M.~Mohri, S.~Reddi, S.~Stich, and A.~T. Suresh,
  ``{Scaffold: Stochastic controlled averaging for federated learning},'' in
  \emph{International Conference on Machine Learning}.\hskip 1em plus 0.5em
  minus 0.4em\relax PMLR, 2020, pp. 5132--5143.

\bibitem{khaled2019first}
A.~Khaled, K.~Mishchenko, and P.~Richt{\'a}rik, ``{First analysis of local gd
  on heterogeneous data},'' \emph{arXiv preprint arXiv:1909.04715}, 2019.

\bibitem{konevcny2016federated}
J.~Kone{\v{c}}n{\`y}, H.~B. McMahan, F.~X. Yu, P.~Richt{\'a}rik, A.~T. Suresh,
  and D.~Bacon, ``Federated learning: Strategies for improving communication
  efficiency,'' \emph{arXiv preprint arXiv:1610.05492}, 2016.

\bibitem{li2015global}
G.~Li and T.~K. Pong, ``Global convergence of splitting methods for nonconvex
  composite optimization,'' \emph{SIAM Journal on Optimization}, vol.~25,
  no.~4, pp. 2434--2460, 2015.

\bibitem{li2016douglas}
------, ``{Douglas--Rachford splitting for nonconvex optimization with
  application to nonconvex feasibility problems},'' \emph{Mathematical
  programming}, vol. 159, no.~1, pp. 371--401, 2016.

\bibitem{li2019convergence}
X.~Li, K.~Huang, W.~Yang, S.~Wang, and Z.~Zhang, ``On the convergence of fedavg
  on non-iid data,'' \emph{arXiv preprint arXiv:1907.02189}, 2019.

\bibitem{li2019convergence[b]}
X.~Li and F.~Orabona, ``On the convergence of stochastic gradient descent with
  adaptive stepsizes,'' in \emph{The 22nd International Conference on
  Artificial Intelligence and Statistics}.\hskip 1em plus 0.5em minus
  0.4em\relax PMLR, 2019, pp. 983--992.

\bibitem{lions1979splitting}
P.-L. Lions and B.~Mercier, ``Splitting algorithms for the sum of two nonlinear
  operators,'' \emph{SIAM Journal on Numerical Analysis}, vol.~16, no.~6, pp.
  964--979, 1979.

\bibitem{mcmahan2017communication}
B.~McMahan, E.~Moore, D.~Ramage, S.~Hampson, and B.~A. y~Arcas,
  ``Communication-efficient learning of deep networks from decentralized
  data,'' in \emph{Artificial intelligence and statistics}.\hskip 1em plus
  0.5em minus 0.4em\relax PMLR, 2017, pp. 1273--1282.

\bibitem{mitra2021linear}
A.~Mitra, R.~Jaafar, G.~Pappas, and H.~Hassani, ``{Linear Convergence in
  Federated Learning: Tackling Client Heterogeneity and Sparse Gradients},''
  \emph{Advances in Neural Information Processing Systems}, vol.~34, 2021.

\bibitem{mohri2019agnostic}
M.~Mohri, G.~Sivek, and A.~T. Suresh, ``Agnostic federated learning,'' in
  \emph{International Conference on Machine Learning}.\hskip 1em plus 0.5em
  minus 0.4em\relax PMLR, 2019, pp. 4615--4625.

\bibitem{ParB14}
N.~Parikh and S.~Boyd, ``Proximal algorithms,'' \emph{Foundations and Trends in
  optimization}, vol.~1, no.~3, pp. 127--239, 2014.

\bibitem{pathak2020fedsplit}
R.~Pathak and M.~J. Wainwright, ``{FedSplit: An algorithmic framework for fast
  federated optimization},'' \emph{Advances in Neural Information Processing
  Systems}, vol.~33, pp. 7057--7066, 2020.

\bibitem{peaceman1955numerical}
D.~W. Peaceman and H.~H. Rachford, Jr, ``{The numerical solution of parabolic
  and elliptic differential equations},'' \emph{Journal of the Society for
  industrial and Applied Mathematics}, vol.~3, no.~1, pp. 28--41, 1955.

\bibitem{reisizadeh2020fedpaq}
A.~Reisizadeh, A.~Mokhtari, H.~Hassani, A.~Jadbabaie, and R.~Pedarsani,
  ``{FedPAQ: A communication-efficient federated learning method with periodic
  averaging and quantization},'' in \emph{International Conference on
  Artificial Intelligence and Statistics}.\hskip 1em plus 0.5em minus
  0.4em\relax PMLR, 2020, pp. 2021--2031.

\bibitem{richtarik2016parallel}
P.~Richt{\'a}rik and M.~Tak{\'a}{\v{c}}, ``Parallel coordinate descent methods
  for big data optimization,'' \emph{Mathematical Programming}, vol. 156,
  no.~1, pp. 433--484, 2016.

\bibitem{sahu2018convergence}
A.~K. Sahu, T.~Li, M.~Sanjabi, M.~Zaheer, A.~Talwalkar, and V.~Smith, ``On the
  convergence of federated optimization in heterogeneous networks,''
  \emph{arXiv preprint arXiv:1812.06127}, vol.~3, p.~3, 2018.

\bibitem{seidman2019control}
J.~H. Seidman, M.~Fazlyab, V.~M. Preciado, and G.~J. Pappas, ``{A
  control-theoretic approach to analysis and parameter selection of
  Douglas--Rachford splitting},'' \emph{IEEE Control Systems Letters}, vol.~4,
  no.~1, pp. 199--204, 2019.

\bibitem{shamir2014communication}
O.~Shamir, N.~Srebro, and T.~Zhang, ``Communication-efficient distributed
  optimization using an approximate newton-type method,'' in
  \emph{International conference on machine learning}.\hskip 1em plus 0.5em
  minus 0.4em\relax PMLR, 2014, pp. 1000--1008.

\bibitem{stich2018local}
S.~U. Stich, ``{Local SGD converges fast and communicates little},''
  \emph{arXiv preprint arXiv:1805.09767}, 2018.

\bibitem{themelis2020douglas}
A.~Themelis and P.~Patrinos, ``{Douglas--Rachford splitting and ADMM for
  nonconvex optimization: Tight convergence results},'' \emph{SIAM Journal on
  Optimization}, vol.~30, no.~1, pp. 149--181, 2020.

\bibitem{tran2021feddr}
Q.~Tran~Dinh, N.~Pham, D.~Phan, and L.~Nguyen, ``{FedDR--randomized
  Douglas-Rachford splitting algorithms for nonconvex federated composite
  optimization},'' \emph{Advances in Neural Information Processing Systems},
  vol.~34, 2021.

\bibitem{wang2018cooperative}
J.~Wang and G.~Joshi, ``{Cooperative SGD: A unified framework for the design
  and analysis of communication-efficient SGD algorithms},'' \emph{arXiv
  preprint arXiv:1808.07576}, 2018.

\bibitem{wang2019adaptive}
S.~Wang, T.~Tuor, T.~Salonidis, K.~K. Leung, C.~Makaya, T.~He, and K.~Chan,
  ``Adaptive federated learning in resource constrained edge computing
  systems,'' \emph{IEEE Journal on Selected Areas in Communications}, vol.~37,
  no.~6, pp. 1205--1221, 2019.

\bibitem{yan2016self}
M.~Yan and W.~Yin, ``Self equivalence of the alternating direction method of
  multipliers,'' in \emph{Splitting Methods in Communication, Imaging, Science,
  and Engineering}.\hskip 1em plus 0.5em minus 0.4em\relax Springer, 2016, pp.
  165--194.

\bibitem{yuan2021federated}
H.~Yuan, M.~Zaheer, and S.~Reddi, ``Federated composite optimization,'' in
  \emph{International Conference on Machine Learning}.\hskip 1em plus 0.5em
  minus 0.4em\relax PMLR, 2021, pp. 12\,253--12\,266.

\bibitem{zhang2021connection}
X.~Zhang and M.~Hong, ``{On the Connection Between FedDyn and FedPD},'' 2021.

\bibitem{zhang2021fedpd}
X.~Zhang, M.~Hong, S.~Dhople, W.~Yin, and Y.~Liu, ``{FedPD: A federated
  learning framework with adaptivity to non-iid data},'' \emph{IEEE
  Transactions on Signal Processing}, vol.~69, pp. 6055--6070, 2021.

\bibitem{zhao2021automatic}
S.~Zhao, L.~Lessard, and M.~Udell, ``{An automatic system to detect equivalence
  between iterative algorithms},'' \emph{arXiv preprint arXiv:2105.04684},
  2021.

\bibitem{ZhaLU21}
------, ``An automatic system to detect equivalence between iterative
  algorithms,'' \emph{arXiv preprint arXiv:2105.04684}, 2021.

\bibitem{zhao2018federated}
Y.~Zhao, M.~Li, L.~Lai, N.~Suda, D.~Civin, and V.~Chandra, ``{Federated
  learning with non-iid data},'' \emph{arXiv preprint arXiv:1806.00582}, 2018.

\end{thebibliography}

\end{document}